\documentclass[]{article}

\usepackage[utf8]{inputenc}
\usepackage{hyperref}
\usepackage{url}
\usepackage{amsfonts}
\usepackage{microtype}
\usepackage{xcolor}
\usepackage{times}
\usepackage{authblk}
\usepackage{amsmath}
\usepackage{amssymb}
\usepackage{amsthm}
\usepackage{algorithm}
\usepackage{algorithmic}
\usepackage{natbib}
\usepackage{tikz}

\usepackage{geometry}
\newcommand{\ind}[1]{\mathbb I(#1)}

\newcommand{\eps}{\varepsilon}
\newcommand{\R}{\mathbb R}
\newcommand{\Orth}{\mathbb O}
\renewcommand{\Pr}{\mathbf P}

\newcommand{\SBM}{\mathrm{SBM}}

\newcommand{\ACSBM}{\mathrm{ACSBM}}

\newcommand{\rank}{\mathrm{rank}}

\newcommand{\diag}{\mathrm{diag}}

\newcommand{\sign}{\mathrm{sign}}
\newcommand{\vone}{\mathbf{1}}

\newcommand{\bernoulli}{\mathrm{Bernoulli}}
\newcommand{\grdpg}{\mathrm{gRDPG}}

\newcommand{\logit}{\mathrm{logit}}

\newcommand{\id}{\mathrm{id}}
\newcommand{\assumption}[1]{\textbf{\textit{(A#1)}}}

\newtheorem{theorem}{Theorem}
\newtheorem{definition}{Definition}
\newtheorem{lemma}{Lemma}

\newtheorem{proposition}{Proposition}

\newtheorem{fact}{Fact}
\newtheorem{remark}{Remark}

\setcitestyle{numbers,square}

\title{Perfect Spectral Clustering with Discrete Covariates}
\author{Jonathan Hehir}
\author{Xiaoyue Niu}
\author{Aleksandra Slavkovi{\'c}}
\affil{Department of Statistics, Pennsylvania State University, University Park, PA}
\date{\today}

\begin{document}

\maketitle

\begin{abstract}
    Among community detection methods, spectral clustering enjoys two desirable properties: computational efficiency and theoretical guarantees of consistency. Most studies of spectral clustering consider only the edges of a network as input to the algorithm. Here we consider the problem of performing community detection in the presence of discrete node covariates, where network structure is determined by a combination of a latent block model structure and homophily on the observed covariates. We propose a spectral algorithm that we prove achieves perfect clustering with high probability on a class of large, sparse networks with discrete covariates, effectively separating latent network structure from homophily on observed covariates. To our knowledge, our method is the first to offer a guarantee of consistent latent structure recovery using spectral clustering in the setting where edge formation is dependent on both latent and observed factors.
\end{abstract}

\section{Introduction}
\label{p2-introduction}

A structural pattern commonly observed in social networks is \emph{homophily}, the tendency for two nodes sharing a certain trait to be more (or sometimes less) likely to form a connection \citep{mcpherson2001birds}. Homophily may occur on any number of traits, observed or latent, and is known to confound problems of causal inference in the social sciences \citep{smith2008social,shalizi2011homophily,goldsmith2013social,lee2021network}. Homophily, meanwhile, lies at the heart of such issues as segregation \citep{shrum1988friendship,henry2011emergence}, job access \citep{ibarra1992homophily}, and political partisanship \citep{huber2017political}, where homophily on observed traits may be the subject of estimation in its own right. In order to fully understand the effects of network patterns like observed homophily, we first need to separate them from further latent network structure.

In the literature on community detection, latent structure is frequently recovered through a clustering process involving only the network edges, reserving node covariates to validate the clustering results in an approach that conflates latent structure with observed structure \citep{peel2017ground}. What we wish to do instead is to separate the latent from the observed structural patterns. To this end, we consider an extension of the stochastic block model (SBM) \citep{holland1983stochastic} that incorporates homophily on observed, discrete node covariates into a generalized linear model (GLM). We define this model, which we call the \emph{additive-covariate SBM (ACSBM)}, in Section~\ref{p2-model}. The model was previously studied by \citet{mele2019spectral} and allows for flexible modeling choices in which latent communities take a block model structure, covariates may or may not depend on community membership, and the effects of homophily may be modeled through a range of link functions. We give an explicit representation of this model as an SBM (Proposition~\ref{thm:acsbm-representation}), which motivates the use of spectral clustering to estimate the latent structure.

In the context of SBMs, spectral clustering is known as a fast method that achieves consistency in community detection down to established recovery thresholds \citep{mcsherry2001spectral,von2007tutorial,rohe2011spectral,lei2015consistency,su2019strong,abbe2020entrywise}. In Section~\ref{p2-procedure} of this work, we propose a computationally efficient spectral algorithm for recovering the latent structure of the ACSBM. Building on techniques from the field of random dot product graphs \citep{young2007random,rubin2017statistical}, we develop new algebraic tools to synthesize latent structure over an ACSBM network partitioned by its covariate data. We are able to prove that our method recovers the latent communities of the ACSBM perfectly for sufficiently large networks with node degree at least polylogarithmic in $n$. Our theoretical analysis is outlined in Section~\ref{acsbm-consistency-results}, with proofs deferred to Appendix~\ref{appendix}, and empirical evidence given in Section~\ref{acsbm-simulations}. We conclude with a discussion of the results, their implications, and future generalizations in Section~\ref{acsbm-discussion}.

\textbf{Related Work.} Community detection with covariates is a very active area of research, with a wide variety of methods for modeling community structure, estimating effects of covariates in edge formation, and recovering community memberships. Studies that demonstrate consistency in community recovery assume a generating process with ground-truth communities. Quite commonly, these generating processes feature conditional independence between covariates and edges, given community memberships \citep[e.g.,][]{binkiewicz2017covariate,deshpande2018contextual,yang2013community,tallberg2004bayesian,newman2016structure,weng2021community}.  In these models, any two nodes belonging to the same latent community have the same connectivity patterns, regardless of their observed covariates.

Explicit separation of latent from observed effects in edge formation is possible in models lacking this conditional independence structure. Such models include \citep[e.g., ][]{hoff2007modeling,handcock2007model,choi2012stochastic,vu2013model,sweet2015incorporating,huang2018pairwise,mele2019spectral,zhang2019node,roy2019likelihood,ma2020universal}, many of which could be considered broader cases of the model we consider. For example, \citep{hoff2007modeling,handcock2007model,ma2020universal} model latent network structure via more general latent position models, which include SBM as a special case. The remainder focus more explicitly on extending SBM but usually allow greater flexibility in the role of covariates, up to and including allowing arbitrary edge covariates. Since working with SBM likelihood is computationally expensive \citep{snijders1997estimation}, many of these studies rely on approximate methods; only a small handful offer methods that scale to large networks and carry a theoretical guarantee of consistent classification. In particular, \cite{huang2018pairwise} provides a consistency guarantee for spectral clustering only when covariates are independent of community membership, and \cite{ma2020universal} provides guarantees only under the assumption of a positive semi-definite latent structure. Our results do not require these assumptions.

By far the most similar paper to ours is \citet{mele2019spectral}, which considers the same model, ACSBM, but under a different spectral estimation method. The main results concern estimation of covariate effects, while we focus on consistency of latent community recovery. Moreover, the results of \cite{mele2019spectral} implicitly rely on strong assumptions about the community structure that we wish to avoid (see Section~\ref{p2-procedure}) and require node degrees of larger order than $\sqrt{n}$. A follow-up paper \citep{mu2020spectral} proposes a modification to the algorithm to improve robustness, but results are limited to the specific case of a single covariate under the identity link, with linear node degree.

\textbf{Contribution.} We propose a novel spectral algorithm that is computationally efficient and yields perfect clustering for sufficiently large ACSBM networks with high probability. We prove this result for networks with node degree at least polylogarithmic in $n$ in which homophily effects are multiplicative on the probabilty of edge formation; empirical results suggest greater generality. To our knowledge, our method is the first to offer a guarantee of consistent latent structure recovery using spectral clustering in the important setting where edge formation is dependent on both latent and observed factors.

\textbf{Notation.} Let $[n] = \{1, \dots, n \}$, with $S_{[n]}$ denoting the set of all permutations $[n] \to [n]$. The function $\ind{\cdot}$ is the indicator function. We represent networks as adjacency matrices, e.g., $Y \in \{0, 1\}^{n \times n}$. The $i$-th row of the matrix $Y$ is denoted $Y_{i*}$, and the $i$-th column $Y_{*i}$. $\vone_n$ denotes a column vector of $n$ ones. We use $\| x \|_2$ to denote the $\ell_2$ norm of a vector $x$, $\|A\|_F$ to denote the Frobenius norm of a matrix, and $\| A \|_2$ to denote the spectral norm of the matrix $A$, i.e., $\| A \|_2 = \sup_{\|x\|_2 = 1} \|A x\|_2$. All functions of matrices are taken element-wise, with the exception of the matrix absolute value, $|A| = \sqrt{A^T A}$. When $n \to \infty$, we write $a_n = o(b_n)$ if $| a_n / b_n | \to 0$; $a_n = \omega(b_n)$ if $| a_n / b_n | \to \infty$; $a_n = O(b_n)$ if $| a_n / b_n | \leq C$ for some $C > 0$ and all $n$; and $a_n = \Theta(b_n)$ if $|a_n / b_n| \in (C_1, C_2)$ for some $C_2 > C_1 > 0$ and all $n$. Finally, we write $X_n = O_P(b_n)$ if for any $\alpha > 0$ there exists a constant $C$ such that $\Pr(|X_n / b_n| > C) < \alpha$ for all large $n$; and $X_n = o_P(a_n)$ if $\Pr(|X_n / a_n| > \eps) \to 0$ for all $\eps > 0$. Further notation is defined in text as needed.

\textbf{Code.} A Python implementation of our proposed method, including simulation code and additional examples, is available at \url{https://github.com/jonhehir/acsbm}.

\section{Network Model and Representation}
\label{p2-model}

The network model we consider is an extension of the popular stochastic block model (SBM) \citep{holland1983stochastic}, which we recall in Definition~\ref{def:acsbm-base-sbm}.

\begin{definition}
\label{def:acsbm-base-sbm}
Conditioned on community membership $\theta \in [K]^n$, the undirected network $Y \sim \SBM(\theta, B)$ is an SBM with edge probabilities $B \in [0,1]^{K \times K}$ if:
$$
Y_{ij} \overset{ind}\sim \bernoulli (B_{\theta_i \theta_j}), \quad i < j .
$$
\end{definition}

The extension we study is what we call the \emph{additive-covariate stochastic block model} (ACSBM), which is also the model studied in \cite{mele2019spectral}. In this setting, we observe a network with $n$ nodes and $K$ communities, along with a set of $M$ discrete covariates. Links are formed independently, depending on community assignments, as in SBM, as well as on covariate similarity, allowing for explicit modeling of homophily based on the observed covariates. Homophily is therefore modeled in a manner similar to exponential random graph models \citep{goodreau2009birds}, with latent structure modeled like SBM. The specific nature of the covariate influence is captured by a known link function $g$. We state a formal definition of this model in Definition~\ref{def:acsbm}.

\begin{definition}
\label{def:acsbm}
For nodes $i \in [n]$, let $\theta_i \in [K]$ denote latent community membership, and let $Z_i \in [L_1] \times \dots \times [L_M]$ be a vector of $M$ discrete, observed covariates. Let $Z = [Z_1 \mid \dots \mid Z_n]^T$. Conditioned on $\theta$ and $Z$, the undirected network $Y \sim \ACSBM(\theta, Z, B, \beta, g)$ is an additive-covariate SBM with covariate effects $\beta \in \R^M$ and known link function $g$ if:
$$
Y_{ij} \overset{ind}\sim \bernoulli \left( g^{-1} \left( B_{\theta_i \theta_j} + \sum_{m=1}^M \beta_m \ind{Z_{im} = Z_{jm}} \right) \right), \quad i < j .
$$
\end{definition}

While the link function $g$ could in principle be any strictly increasing function whose range includes $[0, 1]$, typical choices inspired by similar models include the logit link \citep[e.g.,][]{handcock2007model,choi2012stochastic,roy2019likelihood,ma2020universal}, log link \citep[e.g.,][]{vu2013model,huang2018pairwise}, probit link \citep[e.g.,][]{hoff2007modeling}, or identity link \citep{mu2020spectral}. Choice of link function should be informed by the nature in which covariates are believed to affect edge formation. Our theoretical analysis in Section~\ref{acsbm-consistency-results} employs the log link, in which the effects of observed homophily are multiplicative on the probability of edge formation. Such effects are particularly reasonable to assume in sparse networks, easily interpreted (if estimated), and mimic the form of other popular models like the degree-corrected block model \citep{karrer2011stochastic}.

The ACSBM's combination of independent edges and discrete attributes leads to an important representation result: the ACSBM, which is an extension of SBM, is also in fact a special case of the SBM. Specifically, Proposition~\ref{thm:acsbm-representation} subdivides each latent community by the observed covariates, yielding an SBM over the resulting set of ``subcommunities.'' This generalizes a similar result stated by \citet{mele2019spectral}.

\begin{proposition}
\label{thm:acsbm-representation}
If $Y \sim \ACSBM(\theta, Z, B, \beta, g)$, then $Y$ is equal in distribution to a $(K \tilde L)$-block SBM, namely $Y \overset D= \SBM(\tilde \theta, \tilde B)$ for:
$$
\begin{aligned}
\tilde L &= \prod_{m=1}^M L_m \\
\tilde \theta &= \tilde L (\theta - \vone_n) + \sum_{m=1}^{M-1} \left[ \prod_{m'=m+1}^M L_{m'} \right] (Z_{*m} - \vone_n) + Z_{*M}, \\
\tilde B &= g^{-1} ( B \boxplus \beta_1 I_{L_1} \boxplus \dots \boxplus \beta_P I_{L_M} ),
\end{aligned}
$$
where $g^{-1}$ is taken element-wise, and $A_1 \boxplus A_2 = (A_1 \otimes \vone_{d_2} \vone_{d_2}^T) + (\vone_{d_1} \vone_{d_1}^T \otimes A_2)$ for matrices $A_1 \in \R^{d_1 \times d_1}, A_2 \in \R^{d_2 \times d_2}$.
\end{proposition}

\begin{remark}
$\tilde \theta$ is formed from a bijection from $[K] \times [L_1] \times \dots \times [L_M]$ to $[K \tilde L]$. In an abuse of notation, we will refer to this mapping later in the paper as $\tilde \theta(\cdot, \cdot)$ where for $k \in [K], z \in [L_1] \times \dots \times [L_M]$, $\tilde \theta(k, z) = \tilde L (k-1) + \sum_{m=1}^{M-1} \left[ \prod_{m'=m+1}^M L_{m'} \right] (z_m - 1) + z_M$.
\end{remark}

The proof of Proposition~\ref{thm:acsbm-representation} is constructive and is given in Appendix~\ref{appendix}. This representation result leads to a natural idea: since any ACSBM network is equivalently represented as an SBM, perhaps familiar SBM-fitting methods can be adapted to fit the ACSBM.

\subsection{Random Dot Product Graphs}

Spectral clustering of SBMs has been studied extensively in the context of (generalized) random dot product graphs (RDPGs) \citep{athreya2017statistical,rubin2017statistical}. The class of (g)RDPGs lends itself well to spectral estimation methods, and any binary, undirected, independent-edge network can be formulated as a generalized random dot product graph. In particular, it is well established that SBMs may be represented as gRDPGs \citep{rubin2017statistical}. Below we state the definition of a gRDPG and follow it with a representation result for ACSBM analogous to Proposition~\ref{thm:acsbm-representation}.

\begin{definition}
The matrix $I_{pq} = \diag(I_p, -I_q)$ is the diagonal matrix whose first $p$ diagonal entries are equal to $+1$ and whose remaining $q$ diagonal entries are equal to $-1$. For $x, y \in \R^d$ and some nonnegative integers $p + q = d$, the indefinite inner product of $x$ and $y$ with signature $(p, q)$ is given by $\langle x, y \rangle_{pq} = \langle x, I_{pq} y \rangle = x^T I_{pq} y$. The indefinite orthogonal group with signature $(p, q)$ is given by the set of matrices $\Orth(p, q) = \{ Q \in \R^{d \times d} : Q^T I_{pq} Q = I_{pq} \}$.
\end{definition}

\begin{definition}
Let $F_X$ be a distribution on $\R^d$. We say the undirected network $Y \sim \grdpg(n, F_X)$ is a generalized random dot product graph with signature $(p, q)$ if $X_1, \dots, X_n \overset{iid}{\sim} F_X$, and $Y_{ij} \mid X_1, \dots, X_n \overset{ind}{\sim} \bernoulli(\langle X_i, X_j \rangle_{pq})$ for $i < j$. The variable $X_i$ is referred to as the latent position of the $i$-th node.
\end{definition}

\begin{remark}
When $q = 0$, we say $Y$ is a random dot product graph (without the ``generalized'' qualification) \citep{young2007random}. In this case, $I_{pq} = I$, the indefinite inner product coincides with the usual dot product (i.e., $\langle x, y \rangle_{pq} = \langle x, y \rangle$), and $\Orth(p, q)$ coincides with the familiar group of $p \times p$ orthogonal matrices.
\end{remark}

Both RDPGs and gRDPGs suffer from inherent identifiability issues.\footnote{For a comprehensive approach to the non-identifiability of gRDPGs, see \citet{agterberg2020two}.} In the case of RDPGs, for example, if any set of latent positions is altered by a common orthogonal transformation, the resulting RDPG has the same distribution, since $\langle x, y \rangle = \langle Qx, Qy \rangle$ for any orthogonal $Q$. In gRDPGs, latent positions can only be identified up a common indefinite orthogonal transformation \citep{rubin2017statistical}. Unlike orthogonal transformations, indefinite orthogonal transformations do not preserve distances or angles, rendering them more burdensome to work with. In the following proposition, we choose our canonical latent positions based on a spectral decomposition, but we clarify that this choice of latent positions is not unique. The proof of Proposition~\ref{thm:acsbm-grdpg-representation} follows as a corollary to Proposition~\ref{thm:acsbm-representation}, based on well known results in the gRDPG literature \citep[e.g.,][Section~2.1]{rubin2017statistical}.

\begin{proposition}
\label{thm:acsbm-grdpg-representation}
If $(\theta_i, Z_i) \in [K] \times [L_1] \times \dots \times [L_M]$ are drawn i.i.d. from a distribution with p.m.f. $\Pr_{\theta, Z}$, and $Y \mid \theta, Z \sim \ACSBM(\theta, Z, B, \beta, g)$ for $Z = [Z_1 \mid \dots \mid Z_n]^T$ and some $\beta \in \R^M$, then $Y$ is equal in distribution to a gRDPG, $Y_{grdpg}$, with latent positions sampled i.i.d. from a mixture of point masses. A canonical distribution for these latent positions is as follows. Let $\tilde B$ as in Proposition~\ref{thm:acsbm-representation}, and let $U_{\tilde B} \Lambda_{\tilde B} U_{\tilde B}^T$ be an eigendecomposition of $\tilde B$. Let $X_{\tilde B} = U_{\tilde B} |\Lambda_{\tilde B}|^{1/2}$, and let $X_{\tilde B}(k, z)$ denote the $\tilde \theta(k, z)$-th row of $X_{\tilde B}$. Let $F_{X_{\tilde B}}$ as follows:
$$
F_{X_{\tilde B}} = \sum_{\substack{k \in [K], \\ z \in [L_1] \times \dots \times [L_M]}} \Pr_{\theta, Z}(\theta = k, Z = z) \delta_{X_{\tilde B}(k, z)} .
$$
Letting $q$ denote the number of negative entries in $\Lambda_{\tilde B}$, we have $Y_{grdpg} \sim \grdpg(n, F_{X_{\tilde B}})$ with signature $(p, q) = (K \tilde L - q, q)$.
\end{proposition}

\section{Proposed Spectral Clustering Procedure}
\label{p2-procedure}

We propose a three-part algorithm (Algorithm~\ref{alg:spectral-acsbm}) to estimate the latent community membership $\theta$ for an ACSBM network. Since an ACSBM with $K$ latent communities is equivalently a $(K \tilde L)$-block SBM per Proposition~\ref{thm:acsbm-representation}, we begin by trying to find the $K \tilde L$ ``subcommunities'' (i.e., $\tilde \theta$) of the SBM representation. Assuming we can recover the $K \tilde L$ subcommunities suitably, the primary remaining challenge is to merge these subcommunities into the original $K$ desired communities (i.e., $\theta$).

This fundamental idea is similar to that underlying \cite{mele2019spectral,mu2020spectral}, but we propose a new method for delineating the subcommunities and matching each subcommunity back to its original latent community, allowing for provably consistent results under mild assumptions. In both \cite{mele2019spectral} and \cite{mu2020spectral}, the process of finding the $K \tilde L$ subcommunities relies only on the expected separation of their spectral embeddings in Euclidean space---a condition not met if any $\beta_m$ is sufficiently small (or zero). Moreover, subsequent estimation of $\beta$ in \cite{mele2019spectral,mu2020spectral} relies implicitly on an assumption that the diagonal entries in $B$ are unique, so that an estimate of $\diag(\tilde B)$ can be clustered into $K$ sets of similar values corresponding to the $K$ latent communities. In contrast, our method is robust to non-significant homophily effects and allows for any choice of $B$ that satisfies a full-rank assumption.

Part~1 of the algorithm essentially seeks to recover $\tilde \theta$ of Proposition~\ref{thm:acsbm-representation}. To do so, we first find adjacency spectral embeddings for the full network. Then we consider each possible covariate configuration $z \in [L_1] \times \dots \times [L_M]$ (of which there are $\tilde L$ total), and cluster the embeddings corresponding to nodes bearing this covariate configuration into $K$ clusters. This yields a set of subcommunities that are each pure in their covariate distribution, since we know that $Z_i \neq Z_j \implies \tilde \theta_i \neq \tilde \theta_j$. A range of clustering methods (e.g., $K$-means) may be used here; existing theory suggests Gaussian mixture models may provide the best finite-sample performance \citep{athreya2016limit,rubin2017statistical}. The computational complexity of Part~1 will depend on the specific clustering method employed.

Part~2 of the algorithm estimates $\tilde B$ so that we may estimate a latent position for each subcommunity. While the embeddings of Part~1 also serve as estimates of latent positions, these estimates are only consistent up to an indefinite orthogonal transformation, which would pose problems for the geometry of Part~3. In practical implementations, Part~2 can be performed in linear time, relative to the number of edges in the network.

Successful clustering in Part~1 of the algorithm implies that we are able to recover $\theta$ up to a permutation for any set of nodes with the same covariates. Part~3 of the algorithm seeks a common permutation for all nodes by attempting to reconcile each covariate configuration with a given reference level (canonically $z = \vone_M$). This is achieved by finding the matching that minimizes the sum of squared distances between estimates of latent positions for each cluster. This optimization is a case of the assignment problem, which can be completed efficiently using the Hungarian algorithm \citep{edmonds1972theoretical}. The computational complexity of Part~3 depends only on $K$ and $\tilde L$. The analysis in Section~\ref{acsbm-consistency-results} assumes these quantities are constant in $n$. If allowed to grow, however, we would only expect consistency of subcommunity recovery (i.e., Part~1) if $K \tilde L$ grew slower than $\sqrt{n}$, based on existing results in SBM recovery \citep[e.g.,][]{lei2015consistency}. Under this assumption, the overall complexity of Part~3 of the algorithm is $o(n^{1.5})$ in time and $o(n)$ in space.

\begin{algorithm}
    \caption{Spectral Clustering of ACSBM}
    \label{alg:spectral-acsbm}
\begin{algorithmic}
    \STATE {\bfseries Input:} adjacency matrix $Y \in \{0, 1\}^{n \times n}$, discrete covariates $Z = [z_1 \mid \dots \mid z_n]^T$, number of latent communities $K$, embedding dimension $d$
    \STATE {\bfseries Output:} estimated block membership $\hat \theta \in [K]^n$
    
    \STATE {}
    \STATE {\textit{\# Part 1: Recover the subcommunities $\tilde \theta$}}
    \STATE Let $\hat X_Y := U |\Lambda|^{1/2}$, where $U \Lambda U^T$ is the truncated eigendecomposition of $Y$ with dimension $d$
    \STATE Let $L_1, \; \dots, \; L_M := \max (Z_{*1}), \; \dots, \; \max (Z_{*M})$
    \FOR{$z \text{ in } [L_1] \times \dots \times [L_M]$}
        \STATE Let $\mathcal I_z := \{ i : z_i = z \}$
        \STATE Let $\hat \theta_z : \mathcal I_z \to [K]$ be a function returning cluster assignments over the rows of $\hat X_Y$ corresponding to the indices $\mathcal I_z$
    \ENDFOR
    
    \STATE {}
    \STATE {\textit{\# Part 2: Estimate $\tilde B$}}
    \FOR{$1 \leq k_1 \leq k_2 \leq K \tilde L$}
        \STATE Let $D_{k_1, k_2} := \{ (i, j) \in [n] \times [n] : i \neq j, \tilde \theta(\hat \theta_{z_i}(i), z_i) = k_1, \tilde \theta(\hat \theta_{z_j}(j), z_j) = k_2 \}$
        \STATE Set $\hat {\tilde B}_{k_1, k_2} = \hat {\tilde B}_{k_2, k_1} := \sum_{(i, j) \in D_{k_1, k_2}} A_{ij} / \max\{ 1, |D_{k_1, k_2}| \}$
    \ENDFOR
    
    \STATE {}
    \STATE {\textit{\# Part 3: Reconcile $\theta$ using $z = \vone_M$ as reference level}}
    \STATE Let $\hat X_{\tilde B}(k, z)$ be the $\tilde \theta(k, z)$-th row of $V |\Psi|^{1/2}$, where $V \Psi V^T$ is an eigendecomposition of $\hat{\tilde B}$
    \FOR{$z \text{ in } [L_1] \times \dots \times [L_M]$}
        \STATE Let $\hat \sigma_z := \arg \min_{\sigma \in S_{[K]}} \sum_{k = 1}^K \| \hat X_{\tilde B}(\sigma(k),z) - \hat X_{\tilde B}(k,\vone_M) \|_2^2$
    \ENDFOR
    
    \STATE{}
    \STATE {\bfseries return} $\hat \theta = [\hat \sigma_{z_i}(\hat \theta_{z_i}(i))]_{i=1}^n$
\end{algorithmic}
\end{algorithm}

\begin{remark}
Algorithm~\ref{alg:spectral-acsbm} takes as input an embedding dimension $d$. This corresponds to the dimension of the latent positions in Proposition~\ref{thm:acsbm-grdpg-representation}, which cannot exceed $K \tilde L$. In the absence of oracle knowledge, this maximum value appears to be a suitable choice for $d$.
\end{remark}

\section{Consistency Results}
\label{acsbm-consistency-results}

Breaking Algorithm~\ref{alg:spectral-acsbm} into its three main parts, we first show that Part~1 consistently recovers $\tilde \theta$ from Proposition~\ref{thm:acsbm-representation}. Next, Part~2 yields a consistent estimate of $\tilde B$, given $\tilde \theta$ from Part~1. Finally, Part~3 yields a consistent estimate of $\theta$, given $\tilde \theta$ from Part~1 and a suitable approximation of $\tilde B$ from Part~2. To make things concrete, we consider the following setting.

\textbf{Setting.} Let $M$ be a positive integer, and let $K, L_1, \dots, L_M$ be integers greater than 1. Let $\Pr_{\theta Z}$ be a probability mass function on $[K] \times [L_1] \times \dots \times [L_M]$. Let $\beta \in \R^M$ be a vector of covariate coefficients and $B_0 \in \R^{K \times K}$ be a symmetric matrix of latent block coefficients. To allow for sparsity, let $\alpha_n \in (0, 1]$ be a sequence controlling the expected degree of our networks. For each $n \geq 1$, we draw $\{ (\theta_i, Z_i) \}_{i=1}^n \in ([K] \times [L_1] \times \dots \times [L_M])^n$ from $(\Pr_{\theta Z})^n$. Letting $B = B_0 + \log(\alpha_n) \vone_K \vone_K^T$, we then draw $Y \mid \theta, Z \sim \ACSBM(\theta, Z, B, \beta, \log)$.

As discussed in Section~\ref{p2-model}, under the log link, the effects of observed homophily are multiplicative on the probability of edge formation. When $\alpha_n \to 0$, this is essentially equivalent to the canonical logit link in the limit, since $\lim_{n \to \infty} \log^{-1}(b + \log(\alpha_n)) / \logit^{-1}(b + \log(\alpha_n)) = 1$ for any constant $b$. We note that in this setting, all edge probabilities scale by $\alpha_n$, so the expected degree of each node is $\Theta(n \alpha_n)$. Although we drop the subscripts, the quantities $\tilde B$ and $X_{\tilde B}$ depend on $n$. When we desire constant quantities, we will use $\alpha_n^{-1} \tilde B$ and $\alpha_n^{-1/2} X_{\tilde B}$.

\textbf{Assumptions.} Our full set of results will require the following assumptions. Assumption \assumption{1} is a relatively standard sparsity constraint in the SBM recovery literature. Assumption \assumption{2} is equivalent to saying the latent SBM structure is full-rank, which is also common. Assumption \assumption{3} requires that each latent community contains a node of each type with nonzero probability.

\begin{itemize}
    \item[\assumption{1}] $\alpha_n = \omega(\log^{4c}n / n)$ for the universal constant $c$ in Lemma~\ref{thm:acsbm-embedding-concentration}.
    \item[\assumption{2}] $\exp(B_0)$ is full-rank.
    \item[\assumption{3}] $\Pr_{\theta Z}(\theta = k, Z = z) > 0$ for all $(k, z) \in [K] \times [L_1] \times \dots \times [L_M]$.
\end{itemize} 

We begin by recasting the ACSBM as a gRDPG with signature $(p, q)$, as prescribed by Proposition~\ref{thm:acsbm-grdpg-representation}. Let $\hat X_Y = U |\Lambda|^{1/2}$ (where $Y \approx U \Lambda U^T$) as in Algorithm~\ref{alg:spectral-acsbm}, and let $\hat X_i$ denote the $i$-th row of $\hat X_Y$ (i.e., the spectral embedding for node $i$). Results from the gRDPG literature tell us that these spectral embeddings will be consistent estimates of the latent positions of the gRDPG, up to an unknown transformation from the indefinite orthogonal group $\Orth(p, q)$. This is stated in Lemma~\ref{thm:acsbm-embedding-concentration}, which follows from \citet[][Theorem~3]{rubin2017statistical}.

\begin{lemma}[\citet{rubin2017statistical}]
\label{thm:acsbm-embedding-concentration}
Under assumptions \assumption{1} and \assumption{3}, there exists a universal constant $c > 1$ and a sequence of matrices $Q \in \Orth(p, q)$ such that:
$$
\max_{i \in [n]} \| Q \hat X_i - X_{\tilde B}(\theta_i, Z_i) \|_2 = O_P \left( \frac{\log^c n}{\sqrt{n}} \right).
$$
\end{lemma}

The uniform consistency of Lemma~\ref{thm:acsbm-embedding-concentration} is the key to Part~1 of the algorithm. In particular, when we look at the spectral embeddings for nodes of a given covariate configuration $z \in [L_1] \times \dots \times [L_M]$, this result yields perfect separation of the embeddings with high probability (Theorem~\ref{thm:acsbm-balls}).

\begin{theorem}
\label{thm:acsbm-balls}
Fix $z \in [L_1] \times \dots \times [L_M]$. Let $\mathcal I_z = \{ i : Z_i = z \}$. Assuming \assumption{1} and \assumption{3}, there exist $K$ sequences of balls $\mathcal B_{1,z}, \dots, \mathcal B_{K,z}$ such that $\hat X_i \in \mathcal B_{\theta_i, z}$ for all $i \in \mathcal I_z$ and $\mathcal B_{1,z}, \dots, \mathcal B_{K,z}$ are disjoint with probability approaching 1.
\end{theorem}

Theorem~\ref{thm:acsbm-balls} is proven in Appendix~\ref{appendix} and is sufficient to support exact recovery of $\tilde \theta$ with high probability under a variety of clustering algorithms, such as $K$-means \citep{lyzinski2014perfect}. However, while Lemma~\ref{thm:acsbm-embedding-concentration} states spherical concentration bounds, the clusters of embeddings generally are not spherical but are asymptotically normal, per the discussion in \citet{rubin2017statistical}. For this reason, Gaussian mixture modeling is often preferred over $K$-means for finite-sample performance \citep{athreya2016limit,rubin2017statistical}.

In view of Theorem~\ref{thm:acsbm-balls}, from here we assume knowledge of $\tilde \theta$ in order to demonstrate consistency in Parts~2 and 3 of the algorithm. Recall that Part~2 of the algorithm estimates $\tilde B$ from Proposition~\ref{thm:acsbm-representation}. While this estimate is not our end goal, we will use this reconstruction of $\tilde B$ to estimate the canonical latent positions $X_{\tilde B}$ from Proposition~\ref{thm:acsbm-grdpg-representation}.

\begin{theorem}
\label{thm:acsbm-tilde-b-consistency}
Let $\hat \theta_z : \mathcal I_z \to [K]$. Suppose for each $z \in [L_1] \times \dots \times [L_M]$, there exists $\tau_z \in S_{[K]}$ such that $\hat \theta_z(i) = \tau_z(\theta_i)$ for all $i \in \mathcal I_z$. Assuming \assumption{1}--\assumption{3}, if $\hat{\tilde B}$ is constructed as in Algorithm~\ref{alg:spectral-acsbm}, then there exists a sequence of $K\tilde L \times K\tilde L$ permutation matrices $T$ such that:
$$
\alpha_n^{-1} \| \hat{\tilde B} - T \tilde BT^{-1} \|_F = o_P \left( \frac{1}{\sqrt{n \log^c n}} \right).
$$
\end{theorem}

Theorem~\ref{thm:acsbm-tilde-b-consistency} follows from the fact that, conditioned on $\tilde \theta$, $\hat{\tilde B}$ is the maximum likelihood estimate for a matrix of SBM probabilities corresponding to the subcommunities of $\tilde \theta$ (up to relabeling). The bounds thus follow from a bit of algebraic manipulation of well-known results \citep{bickel2013asymptotic,tang2022asymptotically}, as outlined in Appendix~\ref{appendix}. Finally, we move on to the main act: reconciling the $\tilde L$ per-covariate clusterings into a single clustering for all nodes.

\begin{theorem}
\label{thm:acsbm-permutation-consistency}
Let $\hat \theta_z : \mathcal I_z \to [K]$ and $\hat X_{\tilde B}(k, z)$ as in Algorithm~\ref{alg:spectral-acsbm}. Suppose for each $z \in [L_1] \times \dots \times [L_M]$, there exists $\tau_z \in S_{[K]}$ such that $\hat \theta_z(i) = \tau_z(\theta_i)$ for all $i \in \mathcal I_z$. Let:
\begin{equation}
\label{eq:acsbm-permutation-opt}
\hat \sigma_z = \arg \min_{\sigma \in S_{[K]}} \sum_{k = 1}^K \| \hat X_{\tilde B}(\sigma(k), z) - \hat X_{\tilde B}(k, \vone_M) \|_2^2 .
\end{equation}
Then, assuming \assumption{1}--\assumption{3}, $\hat \sigma_z(\hat \theta_z(i)) = \tau_{\vone_M}(\theta_i)$ for all $i \in [n]$ with probability approaching 1.
\end{theorem}

Theorem~\ref{thm:acsbm-permutation-consistency} involves an abundance of permutations. We assume that for each covariate configuration $z$, we have a function $\hat \theta_z(\cdot)$ that recovers the values of $\theta_i$ up to a permutation $\tau_z$. We can find such functions with high probability from Part~1 of our algorithm. Then, for each $z$, we estimate a permutation $\hat \sigma_z$ in an attempt to ``reverse'' these permutations. Since the true permutations $\tau_z$ are unknowable, we cannot hope to invert $\tau_z$ exactly. Instead, we seek a permutation that satisfies $\hat \sigma_z \circ \tau_z = \tau_0$ for some common unidentifiable permutation $\tau_0 \in S_{[K]}$. By using $z = \vone_M$ as our reference level, we end up recovering $\tau_0 = \tau_{\vone_M}$.

The proof of Theorem~\ref{thm:acsbm-permutation-consistency} is broken into a number of intermediate results in Appendix~\ref{appendix}, of which we give an overview here. We first consider the task of solving an analog to the matching problem (\ref{eq:acsbm-permutation-opt}) using the true latent positions $X_{\tilde B}$ (Theorem~\ref{thm:acsbm-assignment-true-positions}). A handful of linear algebra reduces this task to an optimization problem over a submatrix of $|\tilde B| = \sqrt{\tilde B \tilde B}$. Analysis of the entries of $|\tilde B|$ is tractable under the log link, as $\tilde B$ decomposes into a chain of Kronecker products (Facts~\ref{thm:exp-kron-combine}, \ref{thm:svds-to-absolute-value-of-kronecker}). Under assumption \assumption{2}, we find that the desired permutation is the unique optimum for the matching problem.

Having shown that the matching problem yields the desired result in the absence of estimation error, it remains to show that the estimation error vanishes asymptotically (Lemma~\ref{thm:acsbm-permutation-estimation-error}). The estimation error is bounded by a multiple of $\| \; |\hat {\tilde B}| - |T \tilde B T^{-1}| \; \|_F$, a bound for which follows from Theorem~\ref{thm:acsbm-tilde-b-consistency}. This, indeed, shrinks to zero faster than the gap between the optimal and second-best matching.

\section{Simulations}
\label{acsbm-simulations}

We evaluate the empirical performance of our method on a variety of sequences of ACSBM networks. First, we consider two sequences of sparse networks ($\alpha_n = n^{-0.8}$) with $K=2$ latent communities and $M=2$ covariates drawn i.i.d. as $\bernoulli(0.5)$. The link function is chosen to be $g = \log$. In the first setting, we use a ``regular'' structure for the latent SBM, $B_0 = 1.5 \, \vone_2 \vone_2^T - I_2$. In the second, we consider something more ``irregular,'' with $B_0 = \vone_2 \vone_2^T + \diag(1, -0.2)$. In both cases, covariate effects are $\beta_1 = 1, \beta_2 = -0.5$. For each of ten values of $n$ ranging from $n=125$ to $n=128000$, we generate 100 networks, then apply Algorithm~\ref{alg:spectral-acsbm}, using Gaussian mixture modeling as our clustering method for Part~1. We calculate a misclassification rate (up to relabeling) as $\min_{\sigma \in S_{[K]}} n^{-1} \sum_{i=1}^n \ind{\sigma(\hat \theta_i) \neq \theta_i}$. The median misclassication rate is plotted in the left panel of Figure~\ref{fig:simulations}, with error bands denoting the interquartile range (IQR). The dashed line represents the worst possible misclassification rate of one half. As we might hope, as $n$ increases, misclassification falls toward zero.

\begin{figure}[b]
    \centering
    \includegraphics[width=5.5in]{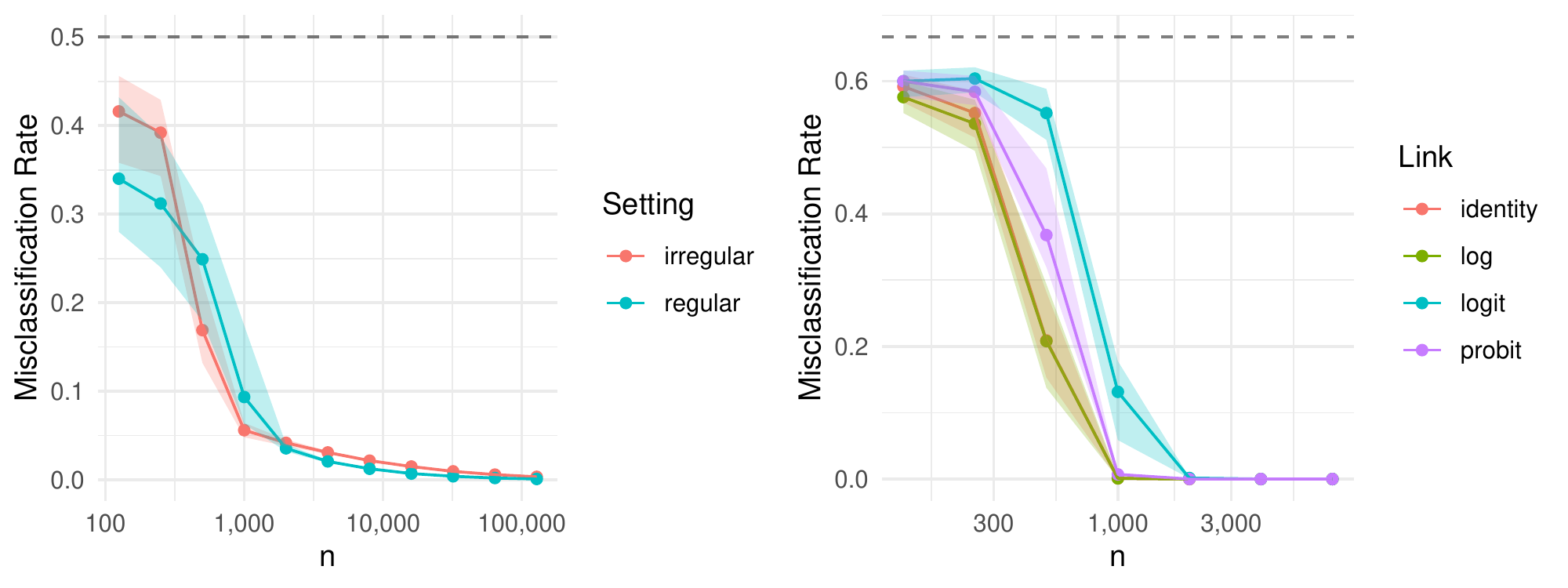}
    \caption{Median proportion (and IQR) of misclassified nodes on repeated simulations of ACSBM models. Left: Sparse settings with $K=2, M=2, g = \log, \alpha_n = n^{-0.8}$. Right: Dense settings with $K=3, M=2$, various $g$, $\alpha_n=1$. Dashed line represents worst possible misclassification ($1-1/K$). Specific parameters given in text.}
    \label{fig:simulations}
\end{figure}

The second set of simulations evaluates the performance of the algorithm on dense networks ($\alpha_n = 1$), with four settings corresponding to different choices of link function: identity, log, logit, and probit. In each case, we model the underlying latent structure as an SBM with $K=3$ communities and model $M=2$ binary covariates, drawn i.i.d. as $\bernoulli(0.5)$. For the identity link, we choose $B = 0.2 \, \vone_3 \vone_3^T -0.1 I_3, \beta_1 = 0.05, \beta_2 = -0.05$. For the remaining links, we use $B = - \vone_3 \vone_3^T -0.5 I_3, \beta_1 = -0.7, \beta_2 = 0.1$. For seven values of $n$ ranging from $n=125$ to $n=8000$, we simulate 100 networks and apply the same clustering methodology as in the previous set of simulations. The results are plotted in the right panel of Figure~\ref{fig:simulations}. Here we see consistency for a greater variety of link functions than was proven in Section~\ref{acsbm-consistency-results}, suggesting even greater generality for our proposed method. In our dense simulations, we achieve perfect clustering in the overwhelming majority of cases when $n \geq 2000$.

We caution against direct comparisons of the simulation settings presented here. For example, in the dense network simulations, one may notice that convergence appears fastest for the log link and slowest for the logit link, but each setting is different in ways that complicate comparisons. While these two settings share the same parameters, the difference in link function subtly affects the relations between entries in $\tilde B$ and leads to a network of lower density for the logit link, since $\logit^{-1}(x) < \log^{-1}(x)$ for any $x \in \R$.

These simulations were conducted on a high performance cluster, but each individual network was simulated and fit using a single CPU core (2.2 GHz Intel Xeon). The most demanding simulation setting was the sparse, regular setting at $n=128000$ nodes, where each network had about 6.2 million edges on average. The average running time for this setting using our Python-based algorithm was 4.35 minutes per network, of which 4.25 minutes were spent in Part~1 of Algorithm~\ref{alg:spectral-acsbm}.

\section{Discussion}
\label{acsbm-discussion}

The task of separating latent from observed structure in networks is critical to a variety of network inference tasks. The method we have proposed is, to our knowledge, the first to offer a rigorous guarantee of consistency of latent structure recovery using spectral clustering in the setting where edge formation is dependent on both observed and latent factors. Our proposed method is computationally efficient and theoretically appealing, using distance in latent space as a means of reconnecting a network partitioned by observed covariates.

While we have focused on estimation of latent community membership $\theta$, we should note that if one wishes to estimate the observed homophily effects $\beta$ of the ACSBM, standard GLM fitting approaches using $\hat \theta$ as a plug-in estimator for $\theta$ yield asymptotically unbiased results under the conditions of Theorem~\ref{thm:acsbm-permutation-consistency}. This follows from the fact that the ACSBM is a special case of the GLM and that $\hat \theta$ is perfect in the limit. Examples demonstrating ACSBM parameter estimation are included in the supplemental code.

We would like to note the limitations of our current work and highlight opportunities for future research. First and foremost, the combinatorial nature of the algorithm restricts its use to discrete covariates. Moreover, since Part~3 of the algorithm estimates permutations over network partitions, any error in permutation selection is likely to introduce considerable error in the final clustering of nodes. A post-processing step akin to spectral clustering with adjustment (SCWA) of \citet{huang2018pairwise} may be useful to avoid finite-sample permutation errors but has yet to be explored. Finally, while we consider only a fixed number of latent communities and covariates, it would be useful to extend our analysis to the case where these quantities grow. Based on existing results for SBM recovery \citep[e.g.,][]{lei2015consistency}, we anticipate the total number of subcommunities of Proposition~\ref{thm:acsbm-representation} is limited to $K \tilde L = o(\sqrt{n})$. It would be interesting, but well outside the scope of this paper, to extend these ideas to a continuous setting, which may alleviate these limitations.

We believe that our proposed method offers promise beyond what has been proven so far. The simulations of Section~\ref{acsbm-simulations} suggest consistency for a wide range of link functions that remains to be rigorously proven. An extension to the degree-corrected setting of \citet{karrer2011stochastic} also seems likely to follow from our current work, based on the geometry of the embeddings of degree-corrected block models and the nature of the matching algorithm, which can be recast as an optimization problem over the \textit{angles} between subcommunities in latent space. An extension for degree correction would greatly expand the practicality of the model we consider, allowing for nodes to exhibit greater variation in node degree, as commonly seen in observed networks, while retaining the simplicity and flexibility of the underlying latent block model structure.

\pagebreak

{
    \small
    \bibliographystyle{plainnat}
    \bibliography{sources}
}

\pagebreak

\appendix

\section{Appendix}
\label{appendix}

\subsection{Preliminaries}

We begin by defining the matrix absolute value and discussing some of its properties.

\begin{definition}
For a matrix $A \in \R^{m \times n}$, we define the matrix absolute value $|A| = \sqrt{A^T A}$. In particular, when $D = \diag(d_1, \dots, d_n)$, we have $|D| = \diag(|d_1|, \dots, |d_n|)$. For symmetric matrices $A = A^T$ with eigendecomposition $A = U \Lambda U^T$, we have $|A| = U |\Lambda| U^T$.
\end{definition}

\begin{fact}
\label{rmk:matrix-absolute-value-unique}
$|A|$ is the unique positive semi-definite square root of $A^T A$.
\end{fact}
\begin{proof}
See \citet[][Theorem~7.3.1]{horn2012matrix}.
\end{proof}

\begin{fact}
\label{thm:matrix-svd-to-absolute-value}
If $A = A^T$ and $A = U \Sigma V^T$ is a singular value decomposition of $A$, then $|A| = U \Sigma U^T$.
\end{fact}
\begin{proof}
We may write $A^T A = A A^T = U \Sigma V^T V \Sigma U^T = U \Sigma^2 U^T$. Note that
$$
U \Sigma U^T \succeq 0 \quad \text{and} \quad  (U \Sigma U^T) (U \Sigma U^T) = A^2 = A^T A.
$$
So by Fact~\ref{rmk:matrix-absolute-value-unique}, $|A| = U \Sigma U^T$ is the unique positive semi-definite square root of $A^T A$.
\end{proof}

\begin{fact}
\label{thm:gram-matrix-of-absolute-value-when-diagonal}
Suppose $A = X D X^T$, where $X^T X$ is diagonal and $D$ is a diagonal matrix with diagonal entries in $\{ \pm 1 \}$. Then $|A| = XX^T$.
\end{fact}
\begin{proof}
Write $A^T A$ as follows:
$$
\begin{aligned}
A^T A &= X D X^T X D X^T \\
    &= X D^2 (X^T X) X^T \quad \text{(diagonals commute)} \\
    &= X X^T X X^T \quad \text{($D^2 = I$)} \\
    &= (X X^T)^2 .
\end{aligned}
$$
Since $XX^T \succeq 0$, $|A| = XX^T$ is the unique positive semi-definite square root of $A^T A$.
\end{proof}

\begin{fact}
\label{thm:absolute-value-orthogonal-conjugation}
If $U$ is orthogonal, then $|UAU^T| = U |A| U^T$.
\end{fact}
\begin{proof}
$$
\begin{aligned}
(U|A|U^T)^2 &= U |A| |A| U^T \\
    &= U A^T A U^T \quad \text{($|A|^ 2= A^T A$)} \\
    &= U A^T U^T U A U^T \\
    &= (U A U^T)^T (U A U^T) .
\end{aligned}
$$
Since $U |A| U^T \succeq 0$, $U |A| U^T$ is the unique positive semi-definite square root of $(U A U^T)^T (U A U^T)$.
\end{proof}

\begin{fact}
\label{thm:absolute-value-of-one-one-plus-id}
Suppose $A = c \vone_n \vone_n^T + d I_n$. Then $|A| = c' \vone_n \vone_n^T + d' I_n$, where:
$$
c' = \frac{|cn + d| - |d|}{n}, \quad d' = |d| .
$$
\end{fact}
\begin{proof}
Let $U \Lambda U^T$ be an eigendecomposition of $\vone_n \vone_n^T$. Then $\Lambda = \diag(n, 0, \dots, 0)$. Now we write an eigendecomposition for $A$:
\begin{equation}
\label{eq:one-one-plus-id}
\begin{aligned}
A &= c \vone_n \vone_n^T + d I_n \\
    &= c U \Lambda U^T + d U U^T \\
    &= U (c \Lambda + d I_n ) U^T .
\end{aligned}
\end{equation}
By definition, then:
$$
|A| = U | c \Lambda + d I_n | U^T ,
$$
which is of the same form as eq. (\ref{eq:one-one-plus-id}), albeit with different constants. The result follows by solving the following for $c'$ and $d'$:
$$
\diag( |cn+d|, |d|, \dots, |d| ) = | c \Lambda + d I_n | = c' \Lambda + d' I_n = \diag(c'n + d', d', \dots, d') .
$$
\end{proof}

\begin{fact}
\label{thm:positive-absolute-value-of-one-one-plus-id}
Suppose $A = c \vone_n \vone_n^T + d I_n$, and $A_{ij} > 0$ for all $i, j \in [n]$. Then $|A|_{ij} > 0$ for all $i, j \in [n]$.
\end{fact}
\begin{proof}
We begin with the trivial cases: If $d \geq 0$, then $A \succeq 0$ and $A = |A|$. Also if $n = 1$, then $A$ is scalar, and $|A|$ is the usual scalar absolute value.

Assume then that $d < 0$ and $n \geq 2$. Let $|A| = c' \vone_n \vone_n^T + d' I_n$ as defined in Fact~\ref{thm:absolute-value-of-one-one-plus-id}. Since all entries in $A$ are positive, then $c > -d = |d|$. Consequently:
$$
cn + d = cn - |d| > |d|n - |d| = |d| (n-1) \geq |d|
$$
As a result, $c'$ must be positive, since $|cn + d| = cn + d > |d|$. Since $d'$ is also positive, every entry in $|A|$ is positive.
\end{proof}

\begin{fact}
\label{thm:matrix-absolute-value-perturbation}
For any two square matrices of equal dimension, $\| \; | A | -  | B | \; \|_F \leq \sqrt{2} \| A - B \|_F$.
\end{fact}
\begin{proof}
See \citet{bhatia2013matrix}, Theorem VII.5.7 and eq. (VII.39).
\end{proof}

We recall our definition of the binary matrix operator $\boxplus$.

\begin{definition}
\label{def:kron-combine}
Let $A \in \R^{m \times m}, B \in \R^{n \times n}$. Then:
$$
A \boxplus B = (A \otimes \vone_n \vone_n^T) + (\vone_m \vone_m^T \otimes B) .
$$
\end{definition}

The operation $\boxplus$ is similar to the more standard Kronecker sum $A \oplus B = (A \otimes I_n) + (I_m \otimes B)$, but with identity matrices replaced by $\vone \vone^T$. Fact~\ref{thm:exp-kron-combine} below also resembles a property that the Kronecker sum satisfies, but replacing the matrix exponential with an element-wise exponential.

\begin{fact}
\label{thm:exp-kron-combine}
For two square matrices $A$ and $B$, $\exp(A \boxplus B) = \exp(A) \otimes \exp(B)$, where $\exp$ is evaluated element-wise.
\end{fact}
\begin{proof}
Observe that the Kronecker product of two square matrices $A \in \R^{m \times m}$ and $B \in \R^{n \times n}$ may be written $A \otimes B = (A \otimes \vone_n \vone_n^T) \odot (\vone_m \vone_m^T \otimes B)$, where $\odot$ denotes the Hadamard product (i.e., element-wise multiplication). From here it follows that:
$$
\begin{aligned}
\exp(A \boxplus B) &= \exp(A \otimes \vone_n \vone_n^T + \vone_m \vone_m^T \otimes B) \\
    &= \exp(A \otimes \vone_n \vone_n^T) \odot \exp(\vone_m \vone_m^T \otimes B) \\
    &= \left( \exp(A) \otimes \vone_n \vone_n^T \right) \odot \left( \vone_m \vone_m^T \otimes \exp(B) \right) \\
    &= \exp(A) \otimes \exp(B) .
\end{aligned}
$$
\end{proof}

In light of the Kronecker representation of $\exp(A \boxplus B)$, we review some facts about Kronecker products and inspect their matrix absolute values.

\begin{fact}
\label{thm:symmetric-kroneckers}
If $A = A^T$ and $B = B^T$, then $A \otimes B = (A \otimes B)^T$.
\end{fact}
\begin{proof}
By \citet[][eq.~4.2.5]{horn1991topics}, $(A \otimes B)^T = A^T \otimes B^T = A \otimes B$.
\end{proof}

\begin{fact}
\label{thm:svds-to-absolute-value-of-kronecker}
Let $A = A^T, B = B^T$ with eigendecompositions $A = U \Lambda U^T, B = V \Psi V^T$. If $C = A \otimes B$, then:
$$
|C| = (U \otimes V) |\Lambda \otimes \Psi| (U \otimes V)^T = |A| \otimes |B|.
$$
\end{fact}
\begin{proof}
We begin by writing SVDs for $A$ and $B$, namely:
$$
\begin{aligned}
A &= U |\Lambda| (\sign(\Lambda) U^T) \\
B &= V |\Psi| (\sign(\Psi) V^T) ,
\end{aligned}
$$
where $\sign(\cdot)$ is taken element-wise. It is easy to verify that $\sign(\Lambda) U^T$ and $\sign(\Psi) V^T$ are indeed orthogonal.

Armed with these decompositions, we may apply \citet[][Theorem~4.2.15]{horn1991topics} to find an SVD for C:
$$
\begin{aligned}
C &= ( U \otimes V ) ( |\Lambda| \otimes |\Psi| ) ( \sign(\Lambda)U^T \otimes \sign(\Psi) V^T) \\
    &= ( U \otimes V ) |\Lambda \otimes \Psi| ( \sign(\Lambda)U^T \otimes \sign(\Psi) V^T)
\end{aligned}
$$
Since $A = A^T$ and $B = B^T$, we have that $C = C^T$ (Fact~\ref{thm:symmetric-kroneckers}). Therefore:
$$
\begin{aligned}
|C| &= (U \otimes V) |\Lambda \otimes \Psi| (U \otimes V)^T \quad \text{(Fact~\ref{thm:matrix-svd-to-absolute-value})} \\
    &= (U \otimes V) (|\Lambda| \otimes |\Psi|) (U \otimes V)^T \\
    &= (U |\Lambda| \otimes V |\Psi|) \otimes (U^T \otimes V^T) \\
    &= (U |\Lambda| U^T) \otimes (V |\Psi| V^T) \\
    &= |A| \otimes |B| .
\end{aligned}
$$
\end{proof}

Finally, we give two useful facts about sums and permutations.

\begin{fact}
\label{thm:sum-of-permuted-products}
Let $x_1, \dots, x_n \in \R$. Then for any $\sigma \in S_{[n]}$:
$$
\sum_{i=1}^n x_i x_{\sigma(i)} \leq \sum_{i=1}^n x_i^2 .
$$
\end{fact}
\begin{proof}
This is an application of Cauchy--Schwarz in disguise:
$$
\begin{aligned}
\left( \sum_{i=1}^n x_i x_{\sigma(i)} \right)^2
    &\leq \left( \sum_{i=1}^n x_i^2 \right) \left( \sum_{i=1}^n x_{\sigma(i)}^2 \right) \\
    &= \left( \sum_{i=1}^n x_i^2 \right)^2 .
\end{aligned}
$$
The final statement comes by taking the square root of both sides.
\end{proof}

\begin{fact}
\label{thm:trace-of-psd-matrix-vs-permutation}
Let $A \in \R^{n \times n}$ such that $A \succeq 0$. Then for any $\sigma \in S_{[n]}$:
$$
\sum_{i=1}^n A_{i \sigma(i)} \leq \sum_{i=1}^n A_{ii} .
$$
Moreover, if $\rank(A) = n$ and $\sigma \neq \id$, the inequality is strict.
\end{fact}
\begin{proof}
Since $A \succeq 0$, let $A = XX^T$. Fix $\sigma \in S_{[n]}$. Then:
$$
\begin{aligned}
\sum_{i=1}^n A_{i \sigma(i)} &= \sum_{i=1}^n e_i^T A e_{\sigma(i)} \\
    &= \sum_{i=1}^n \langle X^T e_i, X^T e_{\sigma(i)} \rangle \\
    \textcircled{a} &\leq \sum_{i=1}^n \| X^T e_i \|  \| X^T e_{\sigma(i)} \| \quad \text{(Cauchy--Schwarz)} \\
    &\leq \sum_{i=1}^n \| X^T e_i \|^2 \quad \text{(Fact~\ref{thm:sum-of-permuted-products})} \\
    &= \sum_{i=1}^n \langle X^T e_i, X^T e_i \rangle \\
    &= \sum_{i=1}^n e_i^T A e_i = \sum_{i=1}^n A_{ii} .
\end{aligned}
$$
If $\sigma \neq \id$, the inequality \textcircled{a} is made strict  when $X$ has linearly independent rows, i.e., when $A$ is full-rank.
\end{proof}

\subsection{Proofs of Results}

\subsubsection{Representation Results}

We prove that ACSBM can be represented as an SBM by explicitly constructing such a representation.

\begin{proof}[Proof of Proposition~\ref{thm:acsbm-representation}]
Consider first the case when $M=1$, i.e., $Z = Z_{*1}$. Every edge is an independent Bernoulli random variable whose probability depends on $(\theta_i, Z_{i1})$ and $(\theta_j, Z_{j1})$. It will be convenient to map these tuples to scalars. Let $\tau(k, \ell) = L_1 (k-1) + \ell$, a bijection from $[K] \times [L_1]$ to $[KL_1]$. Let $\tilde \theta^{(1)} \in [KL_1]^n = (\tau(\theta_i, Z_{1i}))_{i=1}^n$. We will now write the edge probabilities in terms of these new scalar quantities. It can be shown (if a bit tediously) that:
$$
\begin{aligned}
\Pr(Y_{ij} = 1 \mid \tilde \theta^{(1)}_i = t_1, \tilde \theta^{(1)}_j = t_2) &= g^{-1} \left( \, [B \otimes \vone_{L_1} \vone_{L_1}^T + \vone_K \vone_K^T \otimes \beta_1 I_{L_1} ]_{t_1 t_2} \, \right) \\
    &= \left[ \, g^{-1}(B \boxplus \beta_1 I_{L_1}) \, \right]_{t_1 t_2} ,
\end{aligned}
$$
where $g^{-1}$ is taken element-wise in the final line. This is precisely the form of the SBM given in Definition~\ref{def:acsbm-base-sbm}. Thus when $M=1$, we can say $Y$ is equal to an SBM with $\tilde L = KL_1$ communities, $\tilde \theta = L_1(\theta - \vone_n) + Z_{*1}$, and edge probabilities $\tilde B = g^{-1}(B \boxplus \beta_1 I_{L_1})$.

The case when $M \geq 2$ follows inductively. Let $Y_1 \sim \ACSBM(\theta, B, Z_1, \beta_1, g) \overset D= \SBM(\tilde \theta^{(1)}, \tilde B^{(1)})$. Define $Y_2 = \ACSBM(\theta, B, [Z_1 \mid Z_2], (\beta_1, \beta_2)^T, g)$. This network is equal in distribution to $Y_2' \sim \ACSBM(\tilde \theta^{(1)}, g(\tilde B^{(1)}), Z_2, \beta_2, g)$. By the $M=1$ case above, these networks are equal in distribution to an SBM with $K L_1 L_2$ communities:
$$
\tilde \theta^{(2)} = L_2 (\tilde \theta^{(1)} - \vone_n) + Z_{*2} = L_2 (L_1 (\theta - \vone_n) + Z_{*1} - \vone_n) + Z_{*2}
$$
and edge probabilities:
$$
g^{-1} \left( \, g(\tilde B^{(1)}) \boxplus \beta_2 I_{L_2} \, \right) = g^{-1}(B \boxplus \beta_1 I_{L_1} \boxplus \beta_2 I_{L_2}) ,
$$
where once again, $g$ and $g^{-1}$ are element-wise.

Proceed inductively to find the forms of $Y_3, \dots, Y_M$, defined analogously to $Y_2$, so that $Y \overset D= Y_M$.
\end{proof}

The gRDPG representation now follows immediately as a corollary.

\begin{proof}[Proof of Proposition~\ref{thm:acsbm-grdpg-representation}]
By Proposition~\ref{thm:acsbm-representation}, we may represent $Y$ as an SBM, i.e., $Y \overset D= \SBM(\tilde \theta, \tilde B)$. The ability to represent an SBM as a gRDPG using latent positions derived from spectral decomposition is a well established practice in the gRDPG literature, e.g., \citet[][Section~2.1]{rubin2017statistical}. Thus Proposition~\ref{thm:acsbm-grdpg-representation} follows as a corollary to Proposition~\ref{thm:acsbm-representation}.
\end{proof}

\subsubsection{Consistency of Part 1}

\begin{proof}[Proof of Theorem~\ref{thm:acsbm-balls}]
By Lemma~\ref{thm:acsbm-embedding-concentration}, we know that:
$$
\max_{i \in [n]} \| Q \hat X_i - X_{\tilde B}(\theta_i, Z_i) \|_2 = O_P \left( \frac{\log^c n}{\sqrt{n}} \right)
$$
for some sequence of matrices $Q \in \Orth(p, q)$. We might prefer a statement in terms of $\hat X_i$, rather than $Q \hat X_i$, which we can make as follows:
$$
\max_{i \in [n]} \| \hat X_i - Q X_{\tilde B}(\theta_i, Z_i) \|_2 \leq \| Q^{-1} \|_2 \left( \max_{i \in [n]} \| Q \hat X_i - X_{\tilde B}(\theta_i, Z_i) \|_2 \right).
$$
We have seemingly done little here but move the troublesome $Q$ and impose an additional nuisance term. However, \citet[][Lemma~5]{rubin2017statistical} states a key result: $\| Q \|_2$ and $ \| Q^{-1} \|_2$ are bounded almost surely. This allows us to eliminate the nuisance term:
$$
\max_{i \in [n]} \| \hat X_i - Q X_{\tilde B}(\theta_i, Z_i) \|_2 = O_P \left( \frac{\log^c n}{\sqrt{n}} \right) .
$$

We still have to grapple with $QX_{\tilde B}$. Observe that for $z$ fixed, the canonical latent positions $X_{\tilde B}(1, z), \dots, X_{\tilde B}(K, z)$ are distinct by construction. Since $Q$ is full-rank, this also applies to $QX_{\tilde B}(1, z), \dots, QX_{\tilde B}(K, z)$. Moreover, in light of the bounded spectral norms of $Q$ and $Q^{-1}$, which bound the singular values of $Q$ in an interval away from zero, the asymptotic distortion of distances is limited. In particular, $\| Q (X_{\tilde B}(k_1, z) - X_{\tilde B}(k_2, z) ) \|_2 = \Theta(\sqrt{\alpha_n})$ almost surely. Combining these facts yields the result, as follows.

Let $\mathcal B(x, r)$ denote a ball centered at $x$ with radius $r$. From our argument above, there exists a sequence of radii $r = O_P(\log^c n/\sqrt{n})$ such that $\hat X_i \in \mathcal B(Q X_{\tilde B}(\theta_i, z), r)$ for all $i \in \mathcal I_z$. Since $\| Q (X_{\tilde B}(k_1, z) - X_{\tilde B}(k_2, z) ) \|_2$ scales with $\sqrt{\alpha_n} = \omega(\log^{2c}n/\sqrt{n})$, these balls shrink in size faster than they converge to the origin. More concretely, let $\mathcal B_{k,z} = \mathcal B(Q X_{\tilde B}(k, z), r)$ for $k \in [K]$. Then for any $k_1, k_2 \in [K]$:
$$
\Pr( \mathcal B_{k_1,z} \cap \mathcal B_{k_2,z} = \emptyset ) = \Pr \left( r < \frac 12 \| QX_{\tilde B}(k_1, z) - QX_{\tilde B}(k_2, z) \|_2 \right) \to 1 ,
$$
since $\| QX_{\tilde B}(k_1, z) - QX_{\tilde B}(k_2, z) \|_2 = \Theta(\sqrt{\alpha_n})$ almost surely, and $r = o_P(\sqrt{\alpha_n})$.
\end{proof}

\subsubsection{Consistency of Part 2}

\begin{proof}[Proof of Theorem~\ref{thm:acsbm-tilde-b-consistency}]
Suppose $Y_{gen} \sim \SBM(\tilde \theta, B_{gen})$ for some symmetric matrix $B_{gen} \in \R^{K \tilde L \times K \tilde L}$. This model is more general than $Y \sim \SBM(\tilde \theta, \tilde B)$. Suppose we have a perfect estimate of $\tilde \theta$ (up to a permutation), and we wish to estimate $B_{gen}$. In this case, the natural approach to estimating $B_{gen}$ via the empirical density of each block is precisely the maximum likelihood estimator, which has been well-studied \citep[e.g.,][]{bickel2013asymptotic}.

Under the theorem hypothesis, we have indeed recovered $\tilde \theta$ up to a permutation of labels. This is true since $\tilde \theta((\tau_{z_i} \circ \hat \theta_{z_i})(i)), z_i) = \tilde \theta_i$ for all $i$, and the function $\tilde \theta(\cdot, \cdot)$ is a bijection. Let $\tau \in S_{[K \tilde L]}$ denote this permutation, and let $T$ denote the corresponding permutation matrix. Then $ T^{-1} \hat {\tilde B} T$ is the maximum likelihood estimator for a model $Y_{gen} \sim \SBM(\tilde \theta, B_{gen})$, and so we may apply the maximum likelihood results of \citet[][Lemma~1]{bickel2013asymptotic} or, more conveniently, \citet[][Theorem~1]{tang2022asymptotically}. Per these results, we can say that for any $k_1, k_2 \in [K \tilde L]$:
$$
n \alpha_n^{-1/2} \left( (T^{-1} \hat{\tilde B} T)_{k_1 k_2} - \tilde B_{k_1 k_2} \right) \overset D\longrightarrow \mathcal N(0, v_{k_1 k_2}),
$$
where $\overset D\longrightarrow \mathcal N(\cdot, \cdot)$ denotes convergence in distribution to the normal distribution, and $v_{k_1 k_2} > 0$ is a constant depending on $k_1$ and $k_2$. In other words:
$$
(T^{-1} \hat{\tilde B} T)_{k_1 k_2} - \tilde B_{k_1 k_2} = O_P \left( \frac{\sqrt{\alpha_n}}{n} \right).
$$
Since $\tilde B$ scales with $\alpha_n$, we rewrite this to be in terms of the constant quantity $\alpha_n^{-1} \tilde B$:
$$
\alpha_n^{-1} \left( (T^{-1} \hat{\tilde B} T)_{k_1 k_2} - \tilde B_{k_1 k_2} \right) = O_P \left( \frac{1}{n \sqrt{\alpha_n}} \right) = o_P \left( \frac{1}{\sqrt{n \log^c n}} \right) .
$$
Since $K$ and $\tilde L$ are kept constant in $n$, these entrywise bounds may be taken as a bound for the Frobenius norm, $\| T^{-1} \hat{\tilde B} T - \tilde B \|_F$. Moreover, since the Frobenius norm is unitarily invariant, we may write:
$$
\| \hat{\tilde B} - T \tilde B T^{-1} \|_F = o_P \left(\frac{1}{\sqrt{n \log^c n}}\right) .
$$
\end{proof}

\subsubsection{Consistency of Part 3}

We first show that the matching problem selects the appropriate permutations in the absence of estimation error, i.e., when applied to the true latent positions $X_{\tilde B}$. Note that the role of the permutation $\sigma$ in Theorem~\ref{thm:acsbm-assignment-true-positions} below differs slightly from its role in Algorithm~\ref{alg:spectral-acsbm}. In the algorithm, there is an unknown permutation that we are looking to reverse for each choice of $z$; in the theorem below, there is no such permutation, so the correct choice of $\sigma$ is the identity permutation.

\begin{theorem}
\label{thm:acsbm-assignment-true-positions}
Assume $Y$ from the setting of Section~\ref{acsbm-consistency-results}. Let $X_{\tilde B}$ as in Proposition~\ref{thm:acsbm-representation}. For any fixed $z \in [L_1] \times \dots \times [L_M]$:
\begin{equation}
\label{eq:acsbm-assignment-optimization-true}
\arg \min_{\sigma \in S_{[K]}} \sum_{k = 1}^K \| X_{\tilde B}(\sigma(k), z) - X_{\tilde B}(k, \vone_M) \|_2^2 = \id .
\end{equation}
Moreover, if $\exp(B)$ is full-rank, $\sigma = \id$ is the unique minimizer.
\end{theorem}
\begin{proof}
To simplify notation for the proof, let $x_{kz} = X_{\tilde B}(k, z)$. We begin by unpacking the squared norm:
$$
\begin{aligned}
\sum_{k = 1}^K \| x_{\sigma(k)z} - x_{k\vone} \|_2^2 &= \sum_{k = 1}^K \langle x_{\sigma(k)z} - x_{k\vone}, x_{\sigma(k)z} - x_{k\vone} \rangle \\
    &= \sum_{k = 1}^K \left( \langle x_{\sigma(k)z}, x_{\sigma(k)z} \rangle + \langle x_{k\vone}, x_{k\vone} \rangle - 2 \langle x_{\sigma(k)z}, x_{k\vone} \rangle \right) \\
    &=  \sum_{k = 1}^K \langle x_{kz}, x_{kz} \rangle + \sum_{k = 1}^K \langle x_{k\vone}, x_{k\vone} \rangle - 2 \sum_{k = 1}^K \langle x_{\sigma(k)z}, x_{k\vone} \rangle
\end{aligned}
$$
Since only the final sum depends on $\sigma$, the optimization problem (\ref{eq:acsbm-assignment-optimization-true}) is equivalent to finding:
$$
\arg \max_{\sigma \in S_{[K]}} \sum_{k = 1}^K \langle x_{\sigma(k)z}, x_{k\vone} \rangle .
$$
Fix $z \in [L_1] \times \dots \times [L_M]$, and let $\tilde B$ as in Proposition~\ref{thm:acsbm-representation}. Next, we will assemble yet another matrix. For any $k_1, k_2 \in [K]$, let $Q_{k_1 k_2} = \langle x_{k_1 z}, x_{k_2 \vone} \rangle$. If we can show that $Q \succeq 0$, the result will follow from Fact~\ref{thm:trace-of-psd-matrix-vs-permutation}. This is our plan. Observe that:
$$
\langle x_{k_1 z}, x_{k_2 \vone} \rangle_{pq} = \tilde B_{\tilde \theta(k_1, z), \tilde \theta(k_2, \vone)} ,
$$
where $(p, q)$ is the signature of the gRDPG corresponding to $Y$. Following from Fact~\ref{thm:gram-matrix-of-absolute-value-when-diagonal}, the inner products that form the entries of $Q$ can be found in $|\tilde B|$, i.e.:
$$
Q_{k_1 k_2} = \langle x_{k_1 z}, x_{k_2 \vone} \rangle = |\tilde B|_{\tilde \theta(k_1, z), \tilde \theta(k_2, \vone)}.
$$
Since $g = \log$, by Fact~\ref{thm:exp-kron-combine}, we can write $\tilde B$ like so:
$$
\tilde B = \exp(B) \otimes \exp(\beta_1 I_{L_1}) \otimes \dots \otimes \exp(\beta_M I_{L_M}).
$$
Lemma~\ref{thm:svds-to-absolute-value-of-kronecker} gives the convenient form of $|\tilde B|$:
$$
|\tilde B| = |\exp(B)| \otimes |\exp(\beta_1 I_{L_1})| \otimes \dots \otimes |\exp(\beta_M I_{L_M})|.
$$
In particular, this means:
$$
\begin{aligned}
Q_{k_1 k_2} &= |\tilde B|_{\tilde \theta(k_1, z), \tilde \theta(k_2, \vone)} \\
    &= |\exp(B)|_{k_1 k_2} \left[ \; |\exp(\beta_1 I_{L_1})| \otimes \dots \otimes |\exp(\beta_M I_{L_M})| \; \right]_{\tilde \theta(1, z), 1} \\
    &= c_z \,  |\exp(B)|_{k_1 k_2} ,
\end{aligned}
$$
where $c_z = \left[ \; |\exp(\beta_1 I_{L_1})| \otimes \dots \otimes |\exp(\beta_M I_{L_M})| \; \right]_{\tilde \theta(1, z), 1}$ is a strictly positive constant. This follows from Fact~\ref{thm:positive-absolute-value-of-one-one-plus-id}, which says that each of the $|\exp(\beta_m I_{L_m})|$ matrices have positive entries. Since $|\exp(B)| \succeq 0$ by construction, we have then that $Q \succeq 0$. Moreover, when $\exp(B)$ is full-rank, $Q \succ 0$.

Applying Fact~\ref{thm:trace-of-psd-matrix-vs-permutation}, we have that $\sigma = \id$ is a solution to our optimization problem; moreover, it is the unique solution when $\exp(B)$ is full-rank.
\end{proof}

Next, we show that the estimation error due to use of $\hat X_{\tilde B}$ in place of $X_{\tilde B}$ vanishes asymptotically. Note that relabeling permutations appear here.

\begin{lemma}
\label{thm:acsbm-permutation-estimation-error}
Assume the conditions of Theorem~\ref{thm:acsbm-permutation-consistency} hold. Let $X_{\tilde B}$ as in Proposition~\ref{thm:acsbm-representation} and $\hat X_{\tilde B}$ as in Algorithm~\ref{alg:spectral-acsbm}. For any fixed $z \in [L_1] \times \dots \times [L_M]$, let:
$$
\begin{aligned}
\hat L_z(\sigma) &= \sum_{k=1}^K \| \hat X_{\tilde B}(\sigma(k), z) - \hat X_{\tilde B}(k, \vone_M) \|_2^2 \\
L_z(\sigma) &= \sum_{k=1}^K \| X_{\tilde B}((\sigma \circ \tau_z)(k), z) - \hat X_{\tilde B}(\tau_{\vone_M}(k), \vone_M) \|_2^2 .
\end{aligned}
$$
Then for any $\sigma_1, \sigma_2 \in S_{[K]}$:
$$
\alpha_n^{-1} ( \hat L_z(\sigma_1) - \hat L_z(\sigma_2) ) = \alpha_n^{-1} ( L_z(\sigma_1) - L_z(\sigma_2) ) + o_P \left(\frac{1}{\sqrt{n \log^c n}}\right) .
$$
\end{lemma}
\begin{proof}
By an argument similar to the proof of Theorem~\ref{thm:acsbm-assignment-true-positions}, we observe that:
$$
\begin{aligned}
\hat L_z(\sigma) &= \hat c_z - 2\sum_{k=1}^K \langle \hat X_{\tilde B}(\sigma(k), z), \hat X_{\tilde B}(k, \vone_M) \rangle \\
L_z(\sigma) &= c_z - 2\sum_{k=1}^K \langle X_{\tilde B}((\sigma \circ \tau_z)(k), z), \hat X_{\tilde B}(\tau_{\vone_M}(k), \vone_M) \rangle \\
\end{aligned}
$$
for some constants $\hat c_z$ and $c_z$. Moreover, continuing to extend the arguments from the proof of Theorem~\ref{thm:acsbm-assignment-true-positions}, we have:
$$
\begin{aligned}
\langle \hat X_{\tilde B}(\sigma(k), z), \hat X_{\tilde B}(k, \vone_M) \rangle &= |\hat {\tilde B}|_{\tilde \theta(\sigma(k), z), \tilde \theta(k, \vone)} \\
\langle X_{\tilde B}((\sigma \circ \tau_z)(k), z), \hat X_{\tilde B}(\tau_{\vone_M}(k), \vone_M) \rangle &= |\tilde B|_{\tilde \theta((\sigma \circ \tau_z)(k), z), \tilde \theta(\tau_{\vone_M}(k), \vone)} \\
    &= (T |\tilde B| T^{-1})_{\tilde \theta(\sigma(k), z), \tilde \theta(k, \vone)} \\
    &= |T \tilde B T^{-1}|_{\tilde \theta(\sigma(k), z), \tilde \theta(k, \vone)},
\end{aligned}
$$
where $T$ is the permutation matrix from Theorem~\ref{thm:acsbm-tilde-b-consistency}. Note that the last line follows from Fact~\ref{thm:absolute-value-orthogonal-conjugation}. Therefore:
$$
\begin{aligned}
&\hat L_z(\sigma_1) - \hat L_z(\sigma_2) - ( L_z(\sigma_1) - L_z(\sigma_2)) \\
&\quad = - 2\sum_{k=1}^K |\hat {\tilde B}|_{\tilde \theta(\sigma_1(k), z), \tilde \theta(k, \vone)} + 2\sum_{k=1}^K |\hat {\tilde B}|_{\tilde \theta(\sigma_2(k), z), \tilde \theta(k, \vone)} \\
&\quad\quad + 2\sum_{k=1}^K |T \tilde B T^{-1}|_{\tilde \theta(\sigma_1(k), z), \tilde \theta(k, \vone)} - 2\sum_{k=1}^K |T \tilde B T^{-1}|_{\tilde \theta(\sigma_2(k), z), \tilde \theta(k, \vone)} \\
&\quad = 2\sum_{k=1}^K \left( |\hat {\tilde B}|_{\tilde \theta(\sigma_2(k), z), \tilde \theta(k, \vone)} - |T \tilde B T^{-1}|_{\tilde \theta(\sigma_2(k), z), \tilde \theta(k, \vone)} \right) \\
&\quad\quad - 2\sum_{k=1}^K \left( |\hat {\tilde B}|_{\tilde \theta(\sigma_1(k), z), \tilde \theta(k, \vone)} - |T \tilde B T^{-1}|_{\tilde \theta(\sigma_1(k), z), \tilde \theta(k, \vone)} \right) .
\end{aligned}
$$
Observe that the final expression consists of $2K$ terms of the form $2(|\hat {\tilde B}|_{ij} - |T \tilde B T^{-1}|_{ij})$. Combining Theorem~\ref{thm:acsbm-tilde-b-consistency} and Fact~\ref{thm:matrix-absolute-value-perturbation}, we know that:
$$
\alpha_n^{-1} \| \; |\hat {\tilde B}| - | T \tilde B T^{-1} | \; \|_F = o_P \left( \frac{1}{\sqrt{n \log^c n}} \right),
$$
from which we claim a bound on the entrywise error for any $i, j \in [K\tilde L]$:
$$
\alpha_n^{-1} ( |\hat {\tilde B}|_{ij} - | T \tilde B T^{-1}|_{ij} ) = o_P \left( \frac{1}{\sqrt{n \log^c n}} \right) .
$$
Summarizing, then, we have:
$$
\alpha_n^{-1} \left( \hat L_z(\sigma_1) - \hat L_z(\sigma_2) - ( L_z(\sigma_1) - L_z(\sigma_2)) \right) = 4K \cdot o_P \left( \frac{1}{\sqrt{n \log^c n}} \right) .
$$
Since $K$ is constant, the final result follows by simple rearrangement.
\end{proof}

For completeness, we end with a formal proof of Theorem~\ref{thm:acsbm-permutation-consistency}.

\begin{proof}[Proof of Theorem~\ref{thm:acsbm-permutation-consistency}]
Let $\hat L_z : S_{[K]} \to \R$ and $L_z : S_{[K]} \to \R$ as in the statement of Lemma~\ref{thm:acsbm-permutation-estimation-error}. We first rewrite the result of Theorem~\ref{thm:acsbm-assignment-true-positions} in a permuted order. For any fixed $z$:
$$
\begin{aligned}
&\arg \min_{\sigma \in S_{[K]}} L_z(\sigma) \\
&\quad = \arg \min_{\sigma \in S_{[K]}} \sum_{k=1}^K \| X_{\tilde B}\left( (\sigma \circ \tau_z)(k), z \right) - X_{\tilde B}\left( \tau_{\vone_M}(k), \vone_M \right) \|_2^2 \\
&\quad = \tau_{\vone_M} \circ \tau_z^{-1} .
\end{aligned}
$$
This follows from the commutativity of the sum and the fact that $S_{[K]}$ is closed under composition. In other words, we may think of the sum as going in order of $\tau_{\vone_M}(1), \dots, \tau_{\vone_M}(K)$ and minimizing over $\sigma \circ \tau_z \in S_{[K]}$ instead, if we prefer, in which case recovering the identity permutation is equivalent to recovering $\sigma \circ \tau_z = \tau_{\vone_M}$.

For each $z$, let $\sigma_z^* = \tau_{\vone_M} \circ \tau_z^{-1}$ denote the optimal permutation, and let:
$$
\begin{aligned}
a_z &= L_z(\sigma_z^*), \\
b_z &= \arg \min_{\sigma \neq \sigma_z^*} L_z(\sigma), \; \text{and} \\
\Delta_z &= b_z - a_z,
\end{aligned}
$$
so that $\Delta_z$ denotes the gap between the optimal and second-best permutation. Let $\Delta_0 = \min_z \Delta_z$. Since $X_{\tilde B}$ scales with $\sqrt{\alpha_n}$, $L_z(\cdot)$ scales with $\alpha_n$, and the quantity $\alpha_n^{-1} \Delta_0$ is constant. By assumption \assumption{2}, we may further assume $\Delta_0 > 0$.

By Lemma~\ref{thm:acsbm-permutation-estimation-error}, we have that for any permutation $\sigma \in S_{[K]}$:
$$
\alpha_n^{-1} ( \hat L_z(\sigma) - \hat L_z(\sigma_z^*) ) = \alpha_n^{-1} ( L_z(\sigma) - L_z(\sigma_z^*) ) + o_P \left(\frac{1}{\sqrt{n \log^c n}}\right) .
$$
We would like these error terms to be less than $\alpha_n^{-1} \Delta_0/2$ for all $z$. Since $\alpha_n^{-1} \Delta_0/2$ is constant, this happens with high probability for sufficiently large $n$. In this case, we have:
$$
\hat \sigma_z = \arg \min_{\sigma \in S_{[K]}} \hat L_z(\sigma) = \arg \min_{\sigma \in S_{[K]}} L_z(\sigma) = \sigma_z^* = \tau_{\vone_M} \circ \tau_z^{-1} .
$$
Consequently, for all $i \in \mathcal I_z$, since $\hat \theta_z(i) = \tau_z(\theta_i)$, we have our desired result:
$$
\hat \sigma_z( \hat \theta_z(i) ) = \tau_{\vone_M}( \tau_z^{-1}( \tau_z(\theta_i) )) = \tau_{\vone_M}(\theta_i).
$$
\end{proof}

\end{document}